\documentclass{article}

\usepackage{times}
\usepackage{graphicx}
\usepackage{subfigure}
\usepackage{natbib}
\usepackage{algorithm}
\usepackage[noend]{algorithmic}
\usepackage{hyperref}

\usepackage{booktabs}

\usepackage[accepted]{icml2019_arxiv}

\icmltitlerunning{Bayes Optimal Early Stopping Policies for Black-Box Optimization}

\usepackage{amsmath,amssymb,amsthm}
\usepackage{url}

\newcommand{\argmax}{\mathrm{argmax}}
\newcommand{\defined}[1]{{\bf #1}}
\newcommand{\ee}{\mbox{ .}}
\newcommand{\E}[1]{\mathbb{E}\left[#1\right]}
\newcommand{\Exp}{\mathbb{E}}

\newcommand{\posintegers}{\mathbb{Z}_+}
\newcommand{\nonnegintegers}{\mathbb{Z}_{\ge 0}}
\renewcommand{\Pr}[1]{\mathbb{P}\left[#1\right]}
\newcommand{\Prob}{\mathbb{P}}
\newcommand{\reals}{\mathbb{R}}
\newcommand{\seq}[1]{\{#1\}}
\newcommand{\set}[1]{\left\{#1\right\}}

\newtheorem{definition}{Definition}
\newtheorem{lemma}{Lemma}
\newtheorem{theorem}{Theorem}

\newenvironment{proofsketch}{%
	\proof}{\endproof}

\DeclareMathAlphabet{\mathpzc}{OT1}{pzc}{m}{it}

\newcommand{\concat}{\frown}
\newcommand{\dist}{\mathcal{D}}
\newcommand{\etts}{c_{\star}}
\newcommand{\feasible}{\mathcal{F}}
\renewcommand{\L}{L}
\newcommand{\observations}{\mathcal{O}}  %
\newcommand{\of}{g}  %
\newcommand{\os}{\mathsf{obs}}  %
\renewcommand{\sp}{q}  %
\newcommand{\pirestart}[1]{\pi_{\mathrm{static}}(#1)}
\newcommand{\seed}{x} %
\newcommand{\seeds}{\feasible}
\newcommand{\success}{\star}  %
\newcommand{\stoppingrules}{\mathcal{T}}
\newcommand{\trace}{\mathsf{trace}}
\newcommand{\tstop}{t_{\mathrm{Stop}}}  %
\newcommand{\T}{\tau}
\newcommand{\tts}{C_{\star}}

\begin{document}

\twocolumn[
\icmltitle{Bayes Optimal Early Stopping Policies \\for Black-Box Optimization}

\begin{icmlauthorlist}
	\icmlauthor{Matthew Streeter}{google}
\end{icmlauthorlist}

\icmlaffiliation{google}{Google Research}

\icmlcorrespondingauthor{Matthew Streeter}{mstreeter@google.com}

\vskip 0.3in
]

\printAffiliationsAndNotice{}

\begin{abstract}
	We derive an optimal policy for adaptively restarting a randomized algorithm, based on observed features of the run-so-far, so as to minimize the expected time required for the algorithm to successfully terminate.
Given a suitable Bayesian prior, this result can be used to select the
optimal black-box optimization algorithm from among a large family
	of algorithms that
	includes random search, Successive Halving, and Hyperband.
On CIFAR-10 and ImageNet hyperparameter tuning problems, the proposed policies offer up to a factor of 13 improvement over random search in terms of expected time to reach a given target accuracy, and up to a factor of 3 improvement over a baseline adaptive policy that terminates a run whenever its accuracy is below-median.
\end{abstract}

\section{Introduction} \label {sec:intro}

Many real-world problems can be effectively solved using black-box
optimization.  Examples include hyperparameter tuning, as well
as design of circuits, antennas, and other structures.
In such problems,
we are given a feasible set $\feasible$, and the goal is to find
a point $x \in \feasible$ that maximizes
an objective function $f: \feasible \rightarrow \reals$, while evaluating
$f$ as few times as possible.

In this work we consider \emph{multi-fidelity} black-box optimization problems
\citep{huang2006sequential}
where, for each point $x \in \feasible$, there is an iterative process that
produces
a sequence of values $\seq{f(x, t)}_{t=1}^T$.  Having
observed $f(x, t)$, we can observe $f(x, t+1)$ by paying a certain
evaluation cost.
For example, in a hyperparameter tuning problem, $f(x, t)$ might be the
validation accuracy obtained after training for $t$ epochs using
hyperparameter vector $x$, and the cost of computing $f(x, t+1)$ having
already computed $f(x, t)$ might be the time required to train for one epoch.
The goal is now to find an $x \in \feasible$ that maximizes $f(x, T)$,
while minimizing total evaluation cost (e.g., total training time).
Our results also apply to the closely-related problem of maximizing
$f(x, t)$ over both $x$ and $t$, a more natural goal in the context of
hyperparameter tuning.
Solving such problems requires addressing the usual challenges associated with
black-box optimization, but also presents the opportunity to reduce cost by
adaptively allocating resources across different values of $x$ based on
observed partial sequences $f(x, 1), f(x, 2), \ldots, f(x, t)$.

Though hyperparameter tuning is perhaps the most common
example of such a problem within machine learning,
the multi-fidelity formulation is also relevant
to more traditional experiment design problems.
For example, in a
circuit design problem, $f(x, 1)$ might be the result of a cheap simulation,
$f(x, 2)$ might be the result of a more expensive one, and $f(x, 3)$ might
be the result of a physical experiment involving the proposed circuit (e.g., see \cite{huang2006sequential}).

In this work, we focus on the resource allocation aspect of multi-fidelity
black-box optimization.  To this end, we assume that points $x \in \feasible$
are sampled from a fixed distribution (which could be uniform or learned),
which in turn induces a distribution over sequences $\seq{f(x, t)}_{t=1}^T$.
We present theoretical results in a Bayesian setting, where the induced distribution
over sequences is given as a prior.
Given the prior, our job is to adaptively determine when to sample new
$x$ values and how to allocate effort among them.  
Experimentally, we show that a simple explore-exploit algorithm can be used
to effectively estimate the prior on-the-fly.

On the surface, the resource allocation aspect of black-box
optimization may seem less interesting than the geometric aspect
(i.e., deciding which $x \in \feasible$ to consider next),
on which most previous work has focused.
However, recent work
has shown that in many cases, a simple resource allocation policy applied to random search can outperform sophisticated Bayesian optimization algorithms \citep{li2017hyperband}.
Thus, even in the restricted setting we consider, improved resource
allocation has significant potential benefit.

The contributions of this paper are twofold.  First, we formulate
an abstract problem in which one may sample sequences
(e.g., accuracy curves) from a known distribution, and observe prefixes
of those sequences by paying a certain cost (e.g., training time).
For this problem, we derive a policy that is optimal in terms of expected time
to reach a success condition (e.g., suitably high accuracy).
This policy has many potential uses beyond the ones already mentioned.
For example, it can be used to adaptively restart a randomized algorithm
(e.g., a SAT solver) based on observed features of the run-so-far,
so as to minimize its expected running time
(e.g., see \cite{gomes1998boosting}).

Second, we show empirically that
this policy can provide order-of-magnitude improvements over random search
and Hyperband on CIFAR-10 and ImageNet
hyperparameter tuning problems, \emph{when provided with an accurate prior}.
Though we do not achieve comparable results without such a prior,
our experiments demonstrate significant headroom which we
hope will motivate future work on this problem.

\section{Related Work}

As a speedup technique for black-box optimization, our work is most closely
related to early stopping methods.  Various methods for early stopping have
been proposed, based on both parametric and non-parametric models
\citep{domhan2015speeding,golovin2017google}.  Recent work on model-free
algorithms such as Successive Halving \citep{jamieson2016non} and Hyperband \citep{li2017hyperband} has shown that algorithms that apply early stopping to random search can be competitive with Bayesian optimization.

Outside of optimization, earlier work demonstrated the potential of restarts
to speed up randomized algorithms such as SAT
solvers \citep{gomes1998boosting}.  In this setting, significant speedups can
be obtained even without adaptivity, using a fixed sequence of restart
thresholds.  The problem of choosing such a sequence has been addressed
in worst-case, online, and average-case settings \citep{luby1993optimal,gagliolo2007learning,streeter2007restart}.
Our work presents adaptive policies that can be applied to the same
problem, offering additional potential speedups.

Finally, our optimal policy is related to
Gittins index policies \citep{gittins1979bandit},
as discussed in \S\ref{sec:gittins}.

\section{Theoretical Results} \label {sec:theory}

We now formalize the resource allocation problem introduced in \S\ref{sec:intro},
define types of policies that can be used to solve it, derive Bayes-optimal
policies, and present algorithms for efficiently computing near-optimal policies.

\subsection{Problem Definition} \label{sec:definitions}

The problem we consider is defined by a tuple $(\seeds, \dist, \observations, \of)$,
where
\begin{itemize}
    \item $\seeds$ is a set of \emph {seeds},
    \item $\dist$ is a probability distribution over $\seeds$,
    \item $\observations$ is a set of possible \emph{observations}, and
    \item $\of: \seeds \times \posintegers \rightarrow \observations$ is an observation function: $\of(\seed, t)$ is what we observe after spending time $t$ on seed $\seed$.
\end{itemize}

We will consider policies that have the ability to sample a seed $\seed \sim \dist$, and to observe $\of(\seed, 1)$ by paying a unit cost.  Once a policy
has already observed $\of(\seed, t)$ for some $\seed$ and $t$, it may observe $\of(\seed, t+1)$ by paying unit
cost.  The goal is to minimize the time required to observe the special symbol
$\success \in \observations$, which indicates that some success condition has
been met.

To simplify the presentation, we have assumed unit observation costs.
However, our results can be readily extended
to costs that depend on $t$ or even on $\seed$, as discussed at the end of \S\ref{sec:theorems}.

For hyperparameter tuning problems, $\seed$ represents a randomly-sampled
hyperparameter vector,
$\of(\seed, t)$ might represent the resulting validation accuracy after
training for $t$ epochs, and $\success$ might
represent validation accuracy above some predetermined threshold.
In the context of speeding
up a randomized SAT solver, $\seed$ represents the seed used for the pseudo-random
number generator, $\of(\seed, t)$ might contain
features based on the solver's internal state after it has run for $t$ time
steps with random seed $\seed$, and $\success$ represents the solver having
terminated successfully.

We consider several types of policies, defined in the next section.
In all cases, executing a policy $\pi$ produces a sequence of
observations, denoted $\trace(\pi)$.  This sequence is random due to
the sampling of seeds from $\dist$, and its distribution is a function of $\pi$.
Let the random variable $\tts(\pi)$ denote
the length of the shortest prefix of $\trace(\pi)$ that contains $\success$.
An optimal policy is one that minimizes the expected cost incurred before observing $\success$:
\[
	\etts(\pi) \equiv \Exp[\tts(\pi)] \ee
\]

\subsubsection{Types of Policies}

We consider multiple types of policies for solving the above problem.

The simplest type of policy is one that repeatedly samples a seed,
then performs a run whose length depends on the observations according
to a fixed adaptive \emph{stopping rule}.

\begin{definition}
A \defined{stopping rule} is a function
$\T: \observations^* \rightarrow \set{0, 1}$ which takes an
observation sequence as input, and returns a boolean indicating whether
to stop making observations.

Executing $\T$ with seed $\seed$ yields the sequence of observations $\os(\T, \seed) \equiv \seq{o_t}_{t=1}^{\tstop}$, where $o_t = g(\seed, t)$, and $\tstop$ is defined by
\[
	\T(o_{1:t}) = \left.
	\begin{cases}
		0 & \text{if } t < \tstop \\
		1 & \text{if } t = \tstop \mbox { .}
	\end{cases}
	\right.
\]
\end{definition}
In the context of hyperparameter tuning, the observations might be accuracy values, and
a possible stopping rule is: \emph{stop if accuracy has not improved in the last 10 time steps}.
We also consider randomized stopping rules, which return a probability rather than a boolean.

\begin{definition}
For any stopping rule $\T$, the \defined{static restart policy}
$\pirestart{\T}$ repeatedly executes $\T$ with independently sampled seeds,
yielding the observation sequence
\[
	\trace(\pirestart{\T}) \equiv \os(\tau, \seed_1) \concat \os(\tau, \seed_2) \concat \ldots
\]
where $\seed_i \sim \dist\ \forall i$, and $\concat$ is the concatenation operator.
\end{definition}

At the opposite extreme, we consider run-switching policies, which have the
ability to suspend and resume individual runs adaptively using an arbitrary
rule.
In the context of hyperparameter tuning, an example might be:
\emph{perform two runs of length $t=10$, each using a random hyperparameter vector, then discard the run with lower accuracy and continue the remaining run indefinitely}.
The recently-developed Hyperband and Successive Halving algorithms can
both be expressed as run-switching policies.

\begin{definition}
A \defined{run-switching policy} $\pi: (\observations^*)^\infty \rightarrow \nonnegintegers$ takes as input an infinite sequence $L$, where $L_i$ is the (possibly empty) sequence of observations for seed $\seed_i$, and
returns the index of the seed to use for the next observation.

Executing $\pi$ yields a random sequence
of observations $\trace(\pi) \equiv \seq{o_t}_{t=1}^\infty$.  Letting $L^t$ denote the input to
$\pi$ on time step $t$, and letting $\seed_i \sim \dist$ be the $i$th sampled seed, $o_t$ and $L^t$ are defined as follows.
\begin{enumerate}
    \item For all $i$, $L^1_i$ is the empty sequence.
    \item $o_t = \of(\seed_i, s + 1)$, where $i = \pi(L^t)$ and $s = |L^t_i|$.
    \item For all $t$, $L^{t+1}_i = L^t_i \concat \langle o_t \rangle$, where $i = \pi(L^t)$,
	    and $L^{t+1}_j = L^t_j$ for $j \neq i$ ($\concat$ denotes concatenation).
\end{enumerate}
\end{definition}

\subsection{Optimal Policies} \label {sec:theorems}

We now derive an optimal run-switching policy.  
Specifically, we will
prove Theorem~\ref{thm:optimal}, which shows that the Bayes-optimal
run-switching
policy is a static restart policy, and that this restart policy repeatedly runs
the stopping rule $\T^*$ that maximizes a certain benefit to cost ratio.

We adopt the following notation.  For any stopping rule $\tau$,
\begin{itemize}
	\item $\sp(\tau) = \Prob_{\seed \sim \dist}[\success \in \os(\tau, \seed)]$ is the probability that a run of $\tau$ succeeds, and
	\item $c(\tau) = \Exp_{\seed \sim \dist}[|\os(\tau, \seed)|]$ is the expected cost of a single run under $\tau$.
\end{itemize}
$\stoppingrules$ is the set of all (possibly randomized) stopping rules.

\begin{theorem} \label {thm:optimal}
The static restart policy $\pi^* \equiv \pirestart{\T^*}$ is an optimal run-switching policy
(i.e., for any run-switching policy $\pi$, $\etts(\pi) \ge \etts(\pi^*)$),
where
\[
	\T^* = \argmax_{\T \in \stoppingrules} \set{\frac {\sp(\tau)} {c(\T)}} \ee
\]
\end{theorem}

In the context of hyperparameter tuning,
Theorem~\ref{thm:optimal}
means that once the optimal policy starts a new training run
it will never revisit a previous one,
meaning that it is not necessary to store
multiple checkpoints or resume a previously paused run in order to execute
the policy.  This also means that the optimal run-switching policy is easy to
parallelize, a significant advantage in practice.

\newcommand{\pirs}{\pi_0}
The proof of Theorem~\ref{thm:optimal} consists of two parts.
Letting $r^* = \frac {\sp(\T^*)} {c(\T^*)}$, we first show that the
static restart policy
$\pirs = \pirestart{\T^*}$ has $\etts(\pirs) = \frac {1} {r^*}$.
We then prove a matching lower bound, showing that any run-switching
policy $\pi$ has $\etts(\pi) \ge \frac {1} {r^*}$.

The first part of the proof is a corollary of the following lemma, which gives
the expected time-to-success of any static restart policy.
The proof
mirrors the proof of Lemma 1 of \citet{luby1993optimal}, which considers
non-adaptive stopping rules defined by an integer time limit.

\begin{lemma} \label{lem:expected_time}
For any stopping rule $\T$, the static restart policy $\pi = \pirestart{\T}$
has expected time-to-success $\etts(\pi) = \frac {c(\T)} {\sp(\T)}$.
\end{lemma}
\begin{proof}
Let $\seed_1 \sim \dist$ be the seed used for the first run of $\T$,
let $C_1 = |\os(\tau, \seed_1)|$ be the cost of the first
run, and let $S$ be the event that the first run succeeds (i.e.,
$\success \in \os(\tau, \seed_1))$.  The first run succeeds with probability
$q = \sp(\T)$.  Conditioned on the first run failing, the expected
remaining time-to-success is $\etts(\pi)$.
Thus, letting $K = \etts(\pi)$, $K$ satisfies the recurrence
\[
	K = q \E{C_1 | S} + (1-q) (\E{C_1 | \lnot S} + K) \ee
\]
Subtracting $K(1-q)$ from both sides,
\[
	K \cdot q = q \E{C_1 | S} + (1-q) \E{C_1 | \lnot S} = \E{C_1} \ee
\]
Thus, $K = \frac {\E{C_1}} {q} = \frac {c(\T)} {\sp(\T)}$, as claimed.
\end{proof}

Because maximizing $\frac {\sp(\T)} {c(\T)}$ is equivalent to minimizing
$\frac {c(\T)} {\sp(\T)}$, Lemma~\ref{lem:expected_time} immediately implies that
the policy given by Theorem~\ref{thm:optimal} is optimal among
\emph{static restart policies}.
To show that $\pirestart{\T^*}$ is also an optimal \emph{run-switching policy},
we now prove the lower bound: $\etts(\pi) \ge \frac {1} {r^*}$.
This is shown in
Lemma~\ref{lem:lower_bound}, the proof of which requires the following lemma.

\begin{lemma} \label {lem:equiv}
For any run-switching policy $\pi$, there
exists a sequence $\seq{\T_j}$ of (randomized) stopping rules such that $\etts(\pi) = \sum_j c(\T_j)$
and $\sp_\pi = \sum_j \sp(\T_j)$, where $\sp_\pi$ is the probability that $\pi$ succeeds
(i.e., $\success \in \trace(\pi)$).
\end{lemma}
\begin{proof}
To define the sequence of stopping rules, suppose we execute $\pi$,
stopping when it succeeds (if ever).
Let $L_j$ be the resulting observation sequence for seed $\seed_j$.
Let $o_j$ be the truncated observation sequence that results from executing $\T_j$.
We will define $\T_j$ in such a way that the random variables $o_j$ and $\L_j$ have exactly the same distribution.

Assuming $o_j$ and $\L_j$ have the same distribution,
\[
c(\T_j) = \E{|o_j|} = \E{|\L_j|} \mbox { .}
\]
Because the cost of running $\pi$ until it succeeds is $\sum_j |\L_j|$, we have $c(\pi) = \E{ \sum_j |\L_j|} = \sum_j \E{|\L_j|} = \sum_j c(\T_j)$.

A similar argument can be used to prove the analogous equation for $\sp$.
Let $S_j$ be the event that $\L_j$ contains the success token $\success$.  Because the
success token can appear at most once in $\L$, the events $\set{S_j}$ are mutually exclusive,
and
\[
  \sp_\pi = \sum_j \Pr{S_j} \mbox { .}
\]
Then, because $\L_j$ and $o_j$ have the same distribution, $\sp(\T_j) = \Pr{S_j}$,
so $\sp_\pi = \sum_j \sp(\T_j)$.

To define $\T_j$ formally, for any observation sequence $o$ let $E^j_o$ be the event that $o$ is a prefix of $\L_j$.
Define
\[
  \T_j(o) \equiv \Pr{ |\L_j| > |o| \ | \ E^j_o} \mbox { .}
\]
It then follows inductively that for any $o$, $\Pr{o_j = o} = \Pr{\L_j = o}$,
so $o_j$ and $\L_j$ have the same distribution.
\end{proof}

\begin{lemma} \label {lem:lower_bound}
Any run-switching policy $\pi$ has $\etts(\pi) \ge \frac {1} {r^*}$.
\end{lemma}
\begin{proof}
By Lemma~\ref{lem:equiv}, there exists a sequence $\seq{\T_j}$ of stopping
rules such that $\etts(\pi) = \sum_j c(\T_j) $ and $\sp_\pi = \sum_j \sp(\T_j)$,
where $\sp_\pi$ is the probability that $\pi$ succeeds when run forever.
For any stopping rule $\T$, $\sp(\T) \le r^* c(\T)$.
Thus,
\[
\sp_\pi = \sum_j \sp(\T_j) \le r^* \sum_j c(\T_j) = r^* \etts(\pi) \ee
\]
If $\sp_\pi = 1$, this implies $\etts(\pi) \ge \frac {1} {r^*}$, as required.
If $\sp_\pi < 1$, $\etts(\pi) = \infty$ and the lemma holds trivially.
\end{proof}

The results of this section can be easily generalized to the case where observing
$\of(\seed, t + 1)$ given $\of(\seed, t)$ has a cost that depends on $t$ and $\seed$.
After redefining $\trace(\pi)$ as a sequence of (observation, cost) pairs, and redefining $c$ and
$\etts$ appropriately, the proof of Lemma~\ref{lem:equiv} requires only minor changes,
while the remaining proofs go through as-is.

\subsection{Relationship to Gittins Indices} \label {sec:gittins}

The optimal policy derived in Theorem~\ref{thm:optimal} is in fact
the
\emph{Gittins index policy} for a particular instance of the Bayesian
multi-armed bandit problem.  Establishing this connection shows that, in
addition to minimizing expected time-to-success,
the policy of Theorem~\ref{thm:optimal} maximizes an exponentially-discounted
count of the number of times the success token is observed.

In the Bayesian multi-armed bandit problem, we are given a set of $k$ ``arms",
each of which is a Markov chain with known initial state and transition
probabilities.  At each time step $t$, a policy selects the
index $i_t$ of the arm to pull.  This causes Markov chain $i_t$ to
transition to a new state, and the player receives a corresponding reward $r_t$,
drawn from a known distribution which depends on the current state of arm $i$.
The goal is to maximize the discounted reward, $\sum_t \beta^t r_t$, for
discount factor $\beta$.
The Gittins index theorem \citep{gittins1979bandit} shows that, if each arm $i$
is currently in state $z_i$, the optimal policy selects arm
$\argmax_i \set{G_i(z_i)}$, where $G_i(z)$ is the \emph{Gittins index}
associated
with arm $i$ when it is in state $z$.  To define the Gittins index, let
$\tstop(\T)$ be a random variable equal to the number of steps taken by
stopping rule $\T$.  As discussed by \cite{weber1992gittins}, the
Gittins index can be defined as
\begin{equation} \label {eq:gittins}
	G_i(z) =  \sup_{\T \in \stoppingrules} \set { \frac { \mathbb{E} [\sum_{t=1}^{\tstop(\T)} \beta^t r_t(i, z, \T) ]} {\mathbb{E}[ \sum_{t=1}^{\tstop(\T)} \beta^t ]}  } \mbox { .}
\end{equation}

To relate this to Theorem~\ref{thm:optimal}, suppose we have an infinite
number of arms, where
arm $i$ corresponds to the $i$th sampled seed.
Each arm has
the same Markov chain, which has a state for every
observation sequence that does not include the success token $\success$.
Additionally, there is an absorbing state that is entered once the success
token is observed.  A reward of 1 is obtained
when first entering the absorbing state, and the reward is 0 otherwise.

For $\beta=1$, the denominator of \eqref{eq:gittins} is $c(\T)$ and
the numerator is $\sp(\T)$, so the stopping rule that obtains the
supremum in \eqref{eq:gittins} is the $\T^*$ defined in
Theorem~\ref{thm:optimal}.  With additional work, it can be shown that the
Gittins index policy is equivalent to $\pirestart{\T^*}$.
The Gittins index theorem then shows that, in addition to minimizing expected time-to-success, $\pirestart{\T^*}$ maximizes
discounted cumulative reward when the discount factor is
sufficiently close to 1.

\subsection{Computing an Optimal Policy}

As shown in Theorem~\ref{thm:optimal}, the 
 problem of computing an optimal run-switching
policy can be reduced to the simpler problem of computing
the stopping rule $\T^* = \argmax_{\T \in \stoppingrules} \set {\frac {\sp(\T)} {c(\T)}}$.
We now show that $\T^*$ can be computed efficiently using binary search.

\newcommand{\h}{\Delta}

Let $r^* = \max_{\T \in \stoppingrules} \set { \frac {\sp(\T)} {c(\T)} }$.
Each iteration of the binary search algorithm
will guess a value $r$, and check whether $r < r^*$ by solving
the maximization problem:
\begin {equation} \label{eq:delta}
	\h(r) = \max_{\T \in \stoppingrules} \set { \sp(\T) - r \cdot c(\T) }
\end {equation}
This is sufficient to determine whether $r < r^*$, as shown by the following
lemma.

\begin{lemma} \label{lem:search}
$\h(r) > 0$ if and only if $r < r^*$.
\end{lemma}
\begin{proof}
$\h(r) > 0$ iff.\ there exists a stopping rule $\T$ with
$\sp(\T) - r \cdot c(\T) > 0$, or equivalently $\frac {\sp(\T)} {c(\T)} > r$.
By definition, such a rule exists iff.\ $r < r^*$.
\end{proof}

\begin{algorithm}
\begin{algorithmic}
  \caption{FindStoppingRule}
  \label{alg:binary_search}
   \STATE {\bfseries Parameters:} $\epsilon > 0$.
   \STATE Initialize $L \leftarrow 0$, $U \leftarrow 1$.
   \WHILE {$U > (1+\epsilon) L$}
	\STATE Set $r \leftarrow \frac {U + L} {2}$.
	\STATE Set $\delta \leftarrow \Delta(r)$ (see equation \eqref{eq:delta}).
	\STATE If $\delta > 0$ set $L \leftarrow r$, otherwise set $U \leftarrow r$.
    \ENDWHILE
	\STATE Return $\hat \T \equiv \argmax_{\T \in \stoppingrules} \set {\sp(\T) - L \cdot c(\T)}$
\end{algorithmic}
\end{algorithm}

Pseudocode for the binary search algorithm is given in Algorithm~\ref{alg:binary_search}.  Assuming it takes cost at least 1 to make an observation, we have
$r^* \le 1$.  Thus, the inequality $L < r^* \le U$ holds initially.  By
Lemma~\ref{lem:search}, this invariant is maintained whenever
the algorithm updates $L$ or $U$.
This, together with the fact that the algorithm only terminates once $U \le (1 + \epsilon) L$,
can be used to show that the algorithm returns a stopping rule $\hat \tau$
with $\frac {\sp(\hat \T)} {c(\hat \T)} \ge \frac {r^*} {1+\epsilon}$.
Together with Lemmas \ref{lem:expected_time} and \ref{lem:lower_bound},
this implies $\pirestart{\hat \T}$ has expected time-to-success 
within a factor $1+\epsilon$ of optimal.
With additional work, it can be shown that
Algorithm~\ref{alg:binary_search} terminates in $O(\log(\frac {1} {\epsilon r^*}))$ iterations.

Each iteration of binary search requires evaluating $\h(r)$ for some $r$.
The best way of doing this depends on how the Bayesian prior over
observation sequences is represented.  In the typical case of a uniform
distribution over a collection of sequences collected as training data,
$\h(r)$ can be computed in time linear in the total number of observations,
as described in the next section.

\subsubsection{Stopping rules as trees} \label{sec:trees}

Any deterministic stopping rule can be represented as a rooted tree whose edges are labeled
with observations.  Any path through the tree corresponds to a possible observation sequence,
and the tree has a path for every sequence for which the rule returns 0 (i.e., does not stop).
In a hyperparameter tuning problem, the edges might be labeled with discretized accuracy
values, and the rule would continue training as long as the observed accuracy-curve-so-far 
matches some path in the tree.

Using this representation, we can compute $\Delta(r)$ in linear time.

\begin{lemma} \label{lem:computef}
Given a uniform distribution $\dist$ over observation sequences $o_1, o_2, \ldots, o_k$,
	$\h(r)$ can be computed in time $O(n)$
where $n = \sum_{i=1}^k |o_i|$.
\end{lemma}
\begin{proofsketch}
In terms of its behavior on these $k$ sequences, any stopping rule can be represented as a
subtree of a tree $T$, where $T$ has one root-to-leaf path for each of the $k$ sequences.
The vertices can be assigned weights so that
the quantity $\sp(\T) - r \cdot c(\T)$ equals the sum of the vertex weights.
Computing $\h(r)$ then becomes the problem of computing a
maximum-weight subtree.  This can be done
working backwards from the leaves in $O(n)$ time.
\end{proofsketch}

Theorem~\ref{thm:binary_search} summarizes the results of this section.

\begin{theorem} \label {thm:binary_search}
Given a uniform distribution over observation sequences $o_1, o_2, \ldots, o_k$,
a run-switching policy that is provably within a factor $1+\epsilon$ of optimal can be computed in time
$O(n \log (\frac {1} {\epsilon r^*}))$,
where $n = \sum_{i=1}^k |o_i|$.
\end{theorem}

\subsection{Approximately Optimal Policies} \label{sec:near_optimal}

As discussed in \S\ref{sec:trees}, an optimal
stopping rule can be represented as a tree whose vertices represent
partial observation sequences.  In order for an optimal policy
computed on training data to generalize well, the statistics
for each vertex must be estimated based on a reasonable number
of observation sequences.  To achieve this, it is necessary to define
the observations appropriately.  For example, instead of using
real-valued validation accuracies as observations, one can use
bucketized accuracies.
We can also prune the tree to enforce a minimum sequence count.

It is also possible to use our approach with a non-uniform prior, such as the
parameteric Bayesian model of \citet{domhan2015speeding}.  To make use
of the algorithm described in Lemma~\ref{lem:computef}, we must
approximate the prior by a uniform prior over a fixed set of observation
sequences, which can be done by drawing a large number of curves from
the prior and discretizing them appropriately.  Because the number of
samples is limited only by computational constraints, as opposed to
available data, the accuracy loss due to discretization
can be made very small.

We can also use Algorithm~\ref{alg:binary_search} to compute stopping rules that are
not expressed as trees.
For example, the probability of
stopping can be based on a logistic regression, using features based on
the observations made so far.  To make use of Algorithm~\ref{alg:binary_search},
we only need to provide a subroutine that computes $\Delta(r)$.  This is
a linear reward-maximization problem that can be approximately solved using
standard reinforcement learning techniques (e.g., policy gradient).

\section{Experiments} \label{sec:experiments}

To demonstrate the benefit of the optimal run-switching policy derived in
\S\ref{sec:theory}, we now evaluate it on two real-world hyperparameter tuning
problems.  For each problem, our experiments are
designed to answer the following questions:
\begin{itemize}
    \item How much benefit do adaptive policies provide over simpler alternatives, such as starting a fresh run every $t$ time steps, for optimally chosen $t$?
    \item How close to optimal can we get when we do not have access to the prior distribution over observation sequences?  In particular, how close to optimal is the performance of model-free algorithms such as Successive Halving and Hyperband?
\end{itemize}

The two benchmark problems involve tuning the hyperparameters of
image classification models for
CIFAR-10 \citep{krizhevsky2009learning} and
ImageNet \citep{russakovsky2015imagenet}.  We use a convolutional
neural network based on LeNet \citep{lecun1998gradient} for CIFAR-10, and we use
Inception-v3 \citep{szegedy2016rethinking} for ImageNet.

Both models use the same set of hyperparameters, which are given
in Table~\ref{tab:hyperparameters}.  Each hyperparameter is sampled
from either a uniform or log-uniform distribution over a certain interval.
The intervals were selected to include the values used in the original
Inception-v3 paper, as well as a range of other plausible values.

For each hyperparameter tuning problem, we sampled $n$ hyperparameter vectors
uniformly at random, and used each one to train for $T$ update cycles, where
$n$ was as large as practically possible, and $T$ was a rough estimate
(based on a few initial runs) of the point at which most runs had achieved
their maximum validation accuracy.  On each update cycle, we train for 1000
gradient descent steps with a mini-batch size of 1024, and then evaluate
validation accuracy on a separate held-out dataset.
We used $n=720$ and $T=100$ for
LeNet trained on CIFAR-10,
and $n=128$ and $T=200$ for Inception-v3 trained on ImageNet.

We recorded the validation accuracy curves produced by each run, and used this
data to simulate executing different
policies.  This approach allows for fast evaluation of new policies once
the initial data has been collected, and also reduces
variance due to the fact that all policies are evaluated on the same data.

To make our results easily reproducible, we have included the accuracy curves
used in our experiments in the supplementary material, along with the code
for our algorithms.

\begin{table*}[h]
  \label{tab:hyperparameters}
	\caption{Hyperparameters for CIFAR-10 and ImageNet Experiments.}
  \centering
	\begin{small}
	\begin{sc}
	\begin{tabular}{lllp{2cm}p{2cm}p{3cm}}
    \toprule
		Parameter & Range & Distribution \\
    \midrule
		Dropout & $[0.01, 1]$ & Uniform \\
		Label smoothing & $[0, 0.25]$ & Uniform \\
		Learning rate (per example) & $[10^{-4}, 10^{-2}]$ & Log-uniform \\
		RMSProp decay & $[0.75, 1]$ & Uniform \\
		RMSProp epsilon & $[1, 10]$ & Log-uniform \\
    \bottomrule
  \end{tabular}
  \end{sc}
  \end{small}
\end{table*}

\subsection{Non-Adaptive Restart Schedules}

Before evaluating the benefit of adaptive restart policies, we first consider
as a baseline the benefit of using a simple restart schedule.
In particular, we consider restarting with a freshly-sampled hyperparameter
vector every $t$ update cycles,
for some fixed $t$.
As shown by \citet{luby1993optimal}, the optimal
restart schedule is of this form.
Figure~\ref{fig:thresholds}
shows the expected time required to reach a given accuracy $a$, for several
different values of $a$, as a function of the restart threshold $t$,
for the CIFAR-10 tuning problem.  The chosen values of $a$
correspond to the 50th, 90th, and 99th percentile accuracies achieved
at the end of a full run.

\begin{figure} [h]
	\begin{center}
  \includegraphics[width=3in]{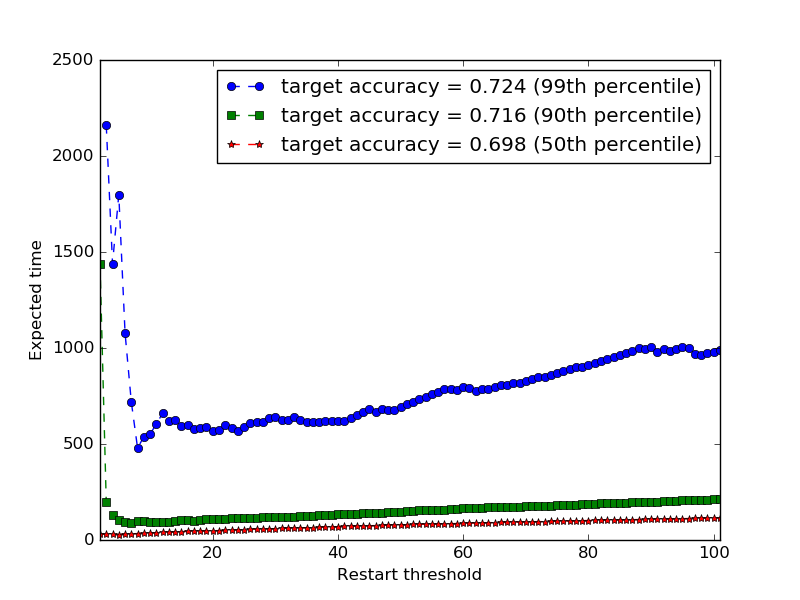}
	\caption{Expected time to reach a given target accuracy
		as a function of restart threshold $t$, when training LeNet
		on CIFAR-10 and restarting training every $t$ update
		cycles using fresh random hyperparameters.}
\label{fig:thresholds}
\end{center}
\end{figure}

As can be seen, choosing $t$ too small can increase the expected time required
to reach a desired accuracy by a large (or even infinite) factor, while
choosing $t$ optimally reduces expected time by a comparatively
small but still non-trivial factor (e.g., roughly a factor of 2 to reach
99th percentile accuracy).

\subsection{Adaptive Run-Switching Policies} \label{sec:adaptive_policies}

We now evaluate the benefit of adaptive run-switching policies
over simple restart schedules.  As discussed in \S\ref{sec:theory},
it suffices to consider restart policies that repeatedly execute a
stopping rule.  We consider stopping rules that, after having
performed a run of length $t$, observe $o_t \in \set{1, 2, \ldots, K}$,
where $o_t$ is the run's current accuracy quantile (relative to other runs of
length $t$ with the same observation prefix), discretized into one of $K$
buckets.  If $K=2$, the stopping
rule makes decisions based on whether accuracy is above or below the
(conditional) median.

\begin{figure} [h]
\begin{center}
  \includegraphics[width=3in]{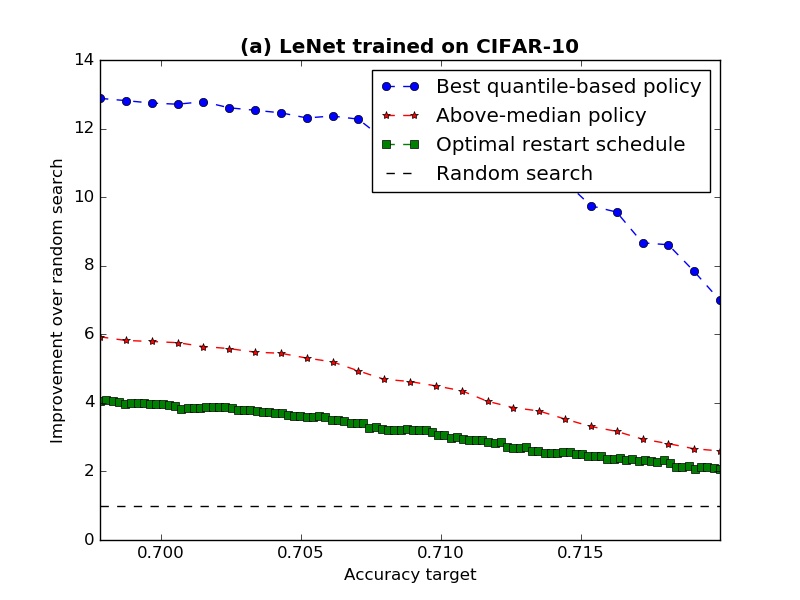}
  \includegraphics[width=3in]{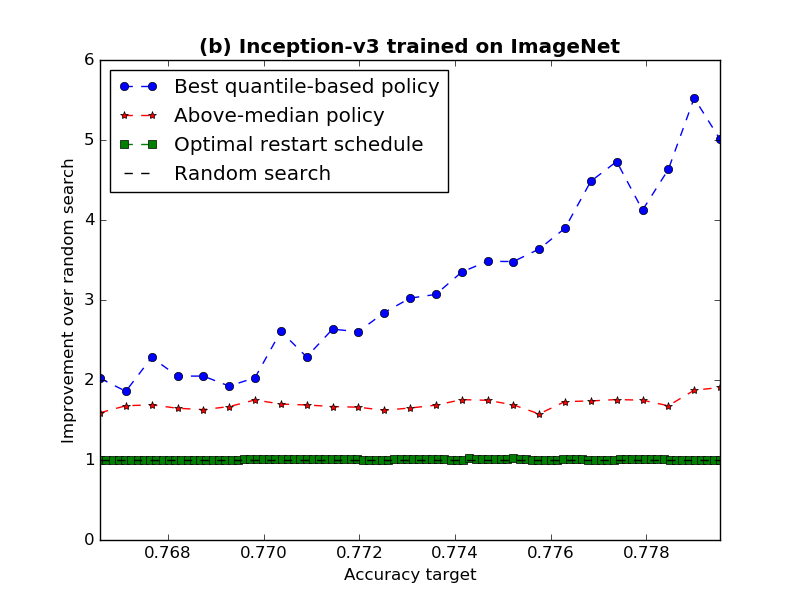}
	\caption{Performance of adaptive run-switching policies, evaluated
	using cross-validation (improvement over random search in
	expected time to reach a given accuracy).}
\label{fig:policies}
\end{center}
\end{figure}

As discussed in \S\ref{sec:near_optimal}, we reduce overfitting by pruning
the policy tree, ensuring that each leaf node is reached by at least 4
runs in the training dataset.  We also considered policies that only branch
when the time so far is a power of 2,
but found this provided no additional benefit over
pruning.

Figure~\ref{fig:policies} shows the improvement over random search that
can be obtained using various schedules and policies.  We plot the
improvement from the best oblivious restart schedule (choosing the best
restart threshold $t$), as well as the improvement from
the best quantile-based policy, optimizing over all $K \in \set{2, 3, 4}$.
As a baseline, we also show the performance of the \emph{above-median stopping rule}, which stops a run at time $t$ if its current accuracy is below the population median.  For all policies that are learned from data, we estimate the
improvement over random search using cross-validation.  To reduce noise
when cross-validating on small dataset, we use a carefully constructed low-variance estimate
described in Appendix A.

As shown in Figure~\ref{fig:policies}, the best quantile-based policy
offers large improvements over random search, and consistently outperforms
both the optimal restart schedule and the above-median policy.  Depending on
the accuracy target, the improvement over random search is up to a factor of
13 for the LeNet model, and up to a factor of 5 for Inception-v3.
In terms of
the expected time to find a
95th-percentile-accuracy hyperparameter vector for Inception-v3,
the best quantile-based policy outperforms the above-median rule by
roughly a factor of 2.5, and outperforms the optimal restart schedule by
more than a factor of 5.

\subsection{Black-Box Optimization Algorithms}

So far we have evaluated the benefit of adaptive policies that were
computed using accuracy curves drawn from the distribution of interest
(and evaluated using cross-validation).
In practice, when facing a new black-box optimization problem we do not know the distribution over accuracy curves, and instead must estimate it on-the-fly.

In this section we compare the performance of four black-box optimization
algorithms.
As baselines, we consider Hyperband \citep{li2017hyperband}, as well as the universal
restart schedule of \citet{luby1993optimal}.  We also consider two new algorithms,
both of which spend half their time collecting data via random search, and the other half exploiting a
policy computed based on that data.  For the \emph{above-median} algorithm, the policy
is the above-median policy described in the previous section.  For the \emph{explore-exploit} algorithm, the policy is the best quantile-based policy,
optimizing over $K \in \set{2,3,4}$ as in \S\ref{sec:adaptive_policies}
and determining the best policy using cross-validation (over the
data collected via random search).  We compute this policy using an
accuracy target equal to the 90th percentile accuracy obtained
during exploration.

To evaluate these algorithms, we ran each algorithm over 4000 times on each
of the two benchmarks, simulating its behavior by
sampling \emph{with replacement} from the collection of pre-recorded
accuracy curves.  Compared with the alternative of actually running
each algorithm on the underlying hyperparameter tuning problem, this approach
allows us to reduce variance by averaging over a much larger number of runs.
It also makes our results easily reproducible given the
accuracy curves, which are included in the supplementary material.

Figure~\ref{fig:algorithms} summarizes the performance of these four algorithms
relative to random
search.
Though Hyperband performs best when tuning the LeNet model
trained on CIFAR-10, it is significantly worse than
random search for tuning Inception-v3 on ImageNet.  In contrast,
the above-median and explore-exploit algorithms
outperform random search on both problems for sufficiently high target accuracies.
As might be expected, both the above-median algorithm and the explore-exploit algorithm perform
better for higher
accuracy targets, where the time available for exploration (and hence the amount of data available
for computing a policy) is larger.

As can be seen by comparing Figures \ref{fig:policies} and \ref{fig:algorithms}, all four algorithms are far from optimal when compared to policies
computed from just a few hundred accuracy curves.
This suggests that substantial gains could be achieved if we could estimate
policies in a more sample-efficient way, for example by using
transfer learning, or by
using policies defined by a function approximator rather than an explicit tree.
The extent to which this is possible is left as an open question for future work.

\begin{figure} [h]
	\begin{center}
  \includegraphics[width=3in]{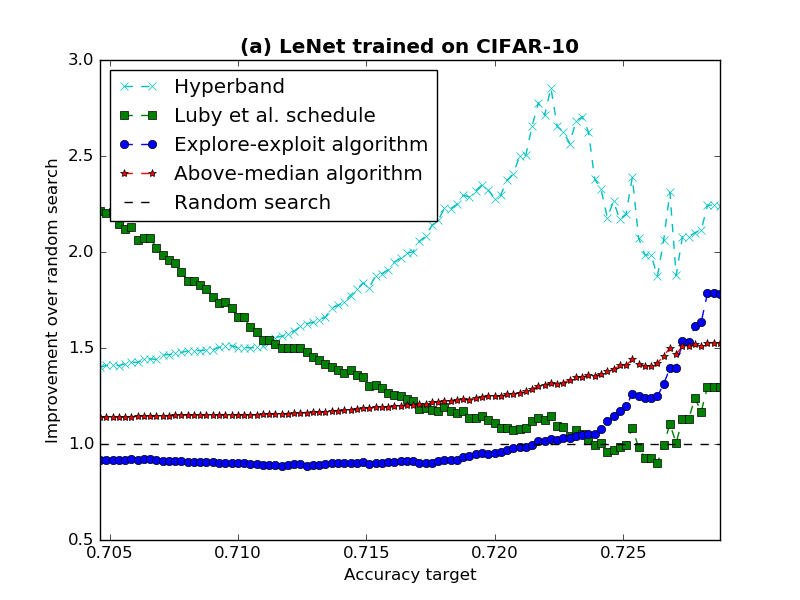}
  \includegraphics[width=3in]{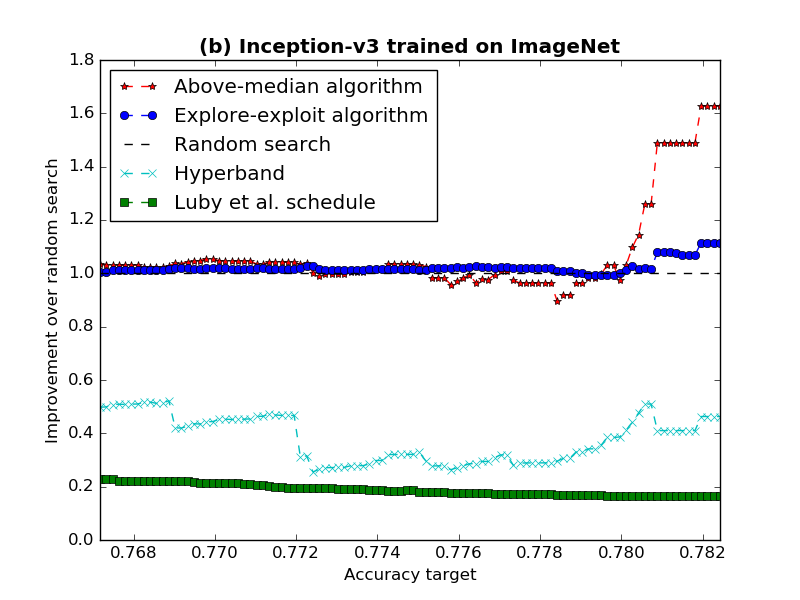}
	\caption{
Performance of black-box optimization algorithms (improvement over random
search in expected time to reach a given accuracy).
		}
\label{fig:algorithms}
\end{center}
\end{figure}

\section{Conclusions}

In this work have have derived optimal early stopping policies applicable
to multi-fidelity black-box optimization, and have evaluated the benefit of
these policies empirically on two hyperparameter tuning problems.
Our main theoretical conclusions are:
\begin{itemize}
	\item {Model-free algorithms for black box optimization, such as Successive Halving and Hyperband, can be viewed as particular \emph{run-switching policies}.}
	\item {The Bayes optimal run-switching policy, in terms of expected time to reach a given accuracy, is a \emph{restart policy} that repeatedly executes the stopping rule that maximizes a certain benefit to cost ratio.  This policy coincides with the Gittins index policy for a related (but different) reward-maximization problem.}
	\item {In contrast to previous early stopping policies,
the optimal policy does not simply stop when it is confident that the current run will not lead to success.  Instead, it stops once it can no longer guarantee a benefit to cost ratio good as that obtained by starting over from scratch.}
\end{itemize}

Empirically, we have found that optimal run-switching policies can offer
order-of-magnitude improvements over random search and
Hyperband, and that
such policies can be estimated using a fairly small number (hundreds)
of observed accuracy curves.  Furthermore, a simple explore-exploit algorithm
based on these policies is already competitive with Hyperband
on our benchmarks, although it fails to deliver the large improvements
over random search that our cross-validation-based analysis shows
are possible given a more accurate prior.

\bibliography{optstop}
\bibliographystyle{icml2019}

\section*{Appendix A}

We now describe the low-variance cross-validation procedure used in our
experiments.

\newcommand{\alg}{\mathcal{A}}
\newcommand{\train}{X}
\newcommand{\test}{Y}

Given an algorithm $\alg$ that produces a static restart policy from training
data,
we evaluate its performance using $k$-fold cross validation, combined
with an important variance-reduction technique which we now describe.

Let $\train_i$ and $\test_i$ denote the training and test
datasets, respectively, for the $i$th split used in cross-validation,
and let $\pi_i = \alg(X_i)$ be the $i$th policy, where $\pi_i = \pirestart{\tau^*_i}$.
Let $\sp_i(\tau^*_i)$ and $c_i(\tau^*_i)$ be the success probability
and
expected cost, respectively, of $\tau^*_i$ as measured on $\test_i$ (both of these depend on the desired target accuracy).
Cross-validation would estimate $\alg$'s expected time-to-success by taking
the average expected time on test data over all splits.  Using Lemma 1, it can be shown that this produces the estimate
\[
	\frac {1} {k} \sum_{i=1}^k \frac {c_i(\tau^*_i)} {\sp_i(\tau^*_i)} \ee
\]

If there are $n$ accuracy curves total, this estimate is asymptotically
unbiased as $\frac n k \rightarrow \infty$.  However, it has high variance
when $k$ is large relative to $n$.
In the extreme case of leave-one-out cross-validation ($k = n$), each $\sp_i(\tau^*_i)$ is estimated based on a single test run, which in general means that at least one $\sp_i(\tau^*_i)$ will be 0, causing the estimate to be infinite independent
of the algorithm $\alg$ that is used to create the policy.

To address this problem, we instead use the estimate
\[
	\frac {\sum_{i=1}^k c_i(\tau^*_i)} {\sum_{i=1}^k \sp_i(\tau^*_i)} \ee
\]
With this estimate, both the numerator and denominator are weighted sums of
$n$ data points, and the estimate has low variance so long as $n$ is large.
Moreover, the bias that remains in our estimate tends
to \emph{understate} the benefit of our adaptive policies (as can be
shown formally using Jensen's inequality).

\end{document}